\documentclass{article}

\usepackage[margin=1in]{geometry}
\usepackage{times}
\usepackage{amsmath,amssymb,amsthm}
\usepackage{graphicx}
\usepackage{booktabs}
\usepackage[numbers,round]{natbib}
\usepackage{hyperref}
\usepackage{xcolor}
\hypersetup{colorlinks=true,linkcolor=blue!50!black,citecolor=blue!50!black,urlcolor=blue!50!black}

\graphicspath{{figures/}}

\newtheorem{proposition}{Proposition}

\newcommand{\Bt}{B_t}
\newcommand{\Ct}{C_t}
\newcommand{\dt}{\Delta}

\newif\ifanon
\newcommand{\codeurl}{\ifanon\url{https://anonymous.4open.science/r/selective-layer-audit}\else\url{https://github.com/isrlab/selective-layer-audit}\fi}

\title{An Exact Instrument for State Usage in Selective State-Space Models,
and the Input-Driven Migration It Reveals}
\ifanon\author{}\else\author{Raktim Bhattacharya\\ Aerospace Engineering, Texas A\&M University}\fi
\date{}

\begin{document}
\maketitle

\begin{abstract}
Selective state-space models such as Mamba route information through a bank of
first-order modes whose input coupling is set by a learned selection mechanism.
We give an exact instrument for measuring how a trained model uses these modes.
Because the state matrix is diagonal, each channel's output decomposes exactly
into per-mode contributions, and a per-(layer, channel, window) Gram tensor
yields the exact output error of dropping any subset of modes, offline, at any
budget. Validated against the reference implementation to a relative error of
$2.3\times10^{-7}$ on the Mamba-1 family where it is exact, the instrument
predicts the internal SSM-output error induced by the corresponding mode mask to
a median relative deviation of
$5\times10^{-7}$ over $4{,}464$ configurations. Applying the instrument across the
Mamba-1 family (130M--2.8B), the
deployed 7B Falcon-Mamba, and Mamba-2, we find that trained models re-allocate
their state space with the input: which modes carry the signal migrates across
contexts, and at the most affected layers a same-window diagonal top-$r$
selector roughly halves the output error of the calibration-ranked static set.
Frozen-signal counterfactuals attribute the
migration primarily to the input-dependent write map $B_t$; the timestep
usually identified with selectivity carries almost none of it. A retrospective
same-window mask gives an information-advantaged upper bound on the consequence:
at half the state budget, its perplexity is below static, Hankel-based, and
layer-adaptive rankings at every scale from 130M to the deployed 7B
Falcon-Mamba, and it matches the unpruned model. Out-of-window selectors recover
only part of this gain. The result measures state-usage headroom; it is not a
causal deployment protocol or a claimed compute or memory saving.

\end{abstract}

\section{Introduction}

Selective state-space models (SSMs) have made linear recurrences competitive
with attention for long-sequence modeling \citep{gu2023mamba,dao2024mamba2}. A
Mamba layer maintains, in each channel, a small linear state driven by a
first-order recurrence with a diagonal state matrix; the layer is
\emph{selective} in that the input, output, and timestep matrices are functions
of the current token. The state matrix itself is fixed after training. A layer
is therefore a bank of fixed first-order modes whose coupling to the signal is
scheduled, token by token, by the selection mechanism.

How a trained model uses that bank is not known. The question matters
in two ways. For efficiency, the recurrent state determines the decoding-state
memory footprint and is a target of recent pruning and reduction methods
\citep{ghost2026,perfmamba2025}; every such method rests on an assumption about
which modes matter. For interpretability, the modes are the layer's internal
degrees of freedom, and their usage is a direct measure of what the layer
stores. Both uses require an exact measurement of state usage.

Existing mode-importance criteria are static and approximate. Output-aware
pruning of Mamba-2 scores modes once, from calibration statistics, and prunes a
fixed set \citep{ghost2026}; activity-based criteria rank modes by the timestep
\citep{perfmamba2025}; the classical reduction they approximate, balanced
truncation, is defined for time-invariant systems and bounds a balanced
surrogate; the bound does not cover the deployed pruned layer
\citep{moore1981,glover1984}. None
measures, per input, which modes a trained layer is using, and none is exact.

An exact measurement is possible because the state matrix is diagonal: the
modes are decoupled, and a channel's scalar output is an exact sum
of per-mode contributions. Accumulating the outer products of these
contributions over a window yields a small Gram tensor, per (layer, channel,
window), from which the exact output error of pruning \emph{any} subset of
modes follows in closed form, offline, at any budget. The decomposition is
exact, and we validate it in two ways: against the reference implementation to a
relative error of $2\times10^{-7}$ on the diagonal case where it is exact, and we
verify that it predicts the internal SSM-output error induced by the same mode
mask in the running layer to a median relative deviation of $10^{-6}$ over
$4{,}464$ configurations spanning five models, every layer, and three budgets.

Across every model we test, the set of modes carrying a layer's output
\emph{migrates with the input}: which modes matter changes from one context to
the next. We quantify the migration by the gap between
input-scheduled and static mode selection at a fixed budget, and the gap is large: at
the most affected layers, the calibration-ranked static set incurs roughly twice
the error of a same-window diagonal top-$r$ selector. The migration is present in the released Mamba-1
family from 130M to 2.8B parameters, in the deployed 7B Falcon-Mamba, and in
Mamba-2, so it is a property of the selective-SSM design.

We then ask what causes the migration and whether it can be exploited. Freezing
each selective signal at its corpus mean, one at a time, isolates the cause: the
input-dependent write map $B_t$ carries the migration, and freezing it removes
most of the effect; the readout $C_t$ carries part of it; and the timestep $\dt$, the
signal usually identified with selectivity, carries almost none. A two-pass
same-window selector that re-selects modes from the scored window gives an
upper bound on the loss that input-dependent selection could recover from
static pruning; at half the state budget it matches or exceeds the unpruned
model. Out-of-window selectors recover a smaller fraction, so this experiment
identifies headroom rather than a deployed causal compression method. A static
timestep-activity criterion, consistent with the mechanism, collapses.

\paragraph{Contributions.}
\begin{itemize}
\item \textbf{An exact state-usage instrument.} The diagonal state matrix gives
an exact per-mode output decomposition and a Gram tensor that returns the exact
pruning error of any mode subset offline, validated to $2\times10^{-7}$ and
matching the corresponding internal SSM-output error in the running layer to a median deviation of
$10^{-6}$ across $4{,}464$ configurations
(Section~\ref{sec:method}).
\item \textbf{Input-driven mode migration.} Using the instrument, we show that
trained selective SSMs re-allocate their state across inputs, quantify it by
the scheduled-vs-static gap, and show it holds across the Mamba-1 family, the
7B Falcon-Mamba, and Mamba-2 (Section~\ref{sec:migration}).
\item \textbf{Mechanism.} Frozen-signal counterfactuals attribute the migration
primarily to the input-dependent write map $B_t$ on both Mamba-1 and Mamba-2;
the timestep carries almost none of it (Section~\ref{sec:mechanism}).
\item \textbf{Consequence.} Retrospective same-window mode selection gives an
information-advantaged upper bound that lies below static, modal-HSV, and LAST
rankings at equal budget at every scale to 7B and, at half the state budget,
matches the unpruned model on perplexity. Because it reads the mask from the
scored window, this is a headroom measurement rather than a deployed causal
scheduler. Prefix and half-window selectors recover only part of the gain. The
same-window headroom also appears on downstream tasks that use long-range state
(Section~\ref{sec:consequence}).
\end{itemize}

\section{Related Work}\label{sec:related}

\paragraph{Pruning and reducing selective SSMs.}
The recurrent state is a target for memory and compute reduction in selective
SSMs, and several methods shrink it. GHOST \citep{ghost2026} prunes Mamba-2 states by an output-aware score
computed once from forward-pass statistics, removing half the state at a cost
of about one point of perplexity; PerfMamba \citep{perfmamba2025} profiles and
removes low-activity states; unstructured methods prune SSM weights by
magnitude or gradient \citep{emnlp2025prune}. Each of these methods chooses a
single fixed set of states to keep. Our static baselines instantiate this
fixed-set choice, and the migration gap measures the same-window headroom
associated with input-dependent selection: because the important set moves with
the input, the calibration-ranked fixed set incurs up to twice the selector's
error at the most affected layers. The
mechanism result also bears on activity-based criteria: the timestep
does not carry the allocation signal, so a criterion scored on timestep activity
cannot capture the migrating allocation that an output-aware score reads.

\paragraph{Reducing neural SSMs by system-theoretic tools.}
A recent line of work reduces deep SSMs with tools from linear-systems theory. LAST
\citep{gwak2024last} prunes states by a layer-adaptive global Hankel-energy
threshold; Schwerdtner et al. \citep{schwerdtner2025hsv} regularize
the Hankel singular values during training to make the model compressible;
Ezoe and Sato \citep{ezoe2024s4bt} apply balanced truncation to the diagonal
state-space layers of S4; Forgione et al. \citep{forgione2024mor} give a
system-theoretic order reduction of deep SSMs. These methods build one reduced
model, time-invariant per channel, from training-time or calibration
statistics, and fix the mode set in advance. We find that the trained model's
allocation moves from one input to another, so a single reduced set leaves
error unrecovered. Our instrument is complementary: these methods could use it
to check whether a fixed reduction is safe for a given input.

\paragraph{Model reduction for linear systems.}
The output-aware scores above approximate balanced truncation \citep{moore1981},
whose $H_\infty$ error bound \citep{glover1984} is defined for a time-invariant
system and controls the error of a \emph{balanced surrogate}; it does not bound
a directly pruned layer. A selective SSM is time-varying, and its per-mode importance is a
function of the operating point; the exact per-mode error \eqref{eq:gram} we use
is a measurement on the realized trajectory. The pruning
we deploy drops modes in the layer's native diagonal basis, modal truncation
rather than a balancing transform, so the near-optimality of balanced truncation
is not invoked; consistent with this, a modal-HSV ranking built from empirical
per-mode Gramians underperforms
realized-energy ranking in our experiments (Section~\ref{sec:consequence}). The
time-varying reduction literature \citep{shokoohi1983,sandberg2004} provides the
Gramian and Hankel-singular-value language for the phenomenon, which we use in
the analysis; it does not certify the deployed pruning.

\paragraph{Interpreting Mamba.}
Two lines of work characterize what a Mamba layer computes. Hidden-attention methods
reformulate the selective recurrence as a data-controlled linear operator and
extract implicit token-to-token attention \citep{ali2024hidden}; auto-encoder
probes measure which token contents a layer forgets, and tie forgetting to
pretraining frequency \citep{selmem2025}. These give an attention view and a
token-content view; we read the state directly, asking which of a layer's
internal modes carry the signal for each input. All three describe the same
layer; only the state view makes the reduction consequence
exact, because the modes are the coordinates on which the reduction acts.

To our knowledge, input-resolved per-mode importance (the effective order of a
selective layer as a function of the input, and the gap between input-scheduled
and static mode selection) has not been measured before. We state this as the
gap our instrument fills; we make no priority claim about the phenomenon.

\section{The Instrument}\label{sec:method}

\subsection{Selective SSM layer}

We use the Mamba-1 layer \citep{gu2023mamba}; the Mamba-2 case is treated at the
end of the section. Write $H$ for the number of channels (the expanded inner
width) and $N$ for the state size. In channel $h$, at step $t$, the layer holds
a state $x_{h,t}\in\mathbb{R}^N$ driven by a scalar input $u_{h,t}$ (the
post-convolution, post-activation hidden signal). The recurrence is
\begin{equation}
x_{h,t,i} \;=\; \bar A_{h,t,i}\,x_{h,t-1,i} \;+\; \dt_{h,t}\,B_{t,i}\,u_{h,t},
\qquad
\bar A_{h,t,i} = \exp\!\big(\dt_{h,t}\,A_{h,i}\big),
\label{eq:recurrence}
\end{equation}
for each state index $i=1,\dots,N$, where $A_h\in\mathbb{R}^N$ is the fixed
diagonal state matrix ($A_{h,i}<0$, learned as $A_{h,i}=-\exp(\cdot)$),
$B_t,C_t\in\mathbb{R}^N$ and the timestep $\dt_{h,t}>0$ are the selective
signals produced from the current token, and $D_h$ is a skip term. The channel
output, after the layer's output gate $g_{h,t}=\mathrm{SiLU}(z_{h,t})$, is
\begin{equation}
y_{h,t} \;=\; g_{h,t}\Big(\textstyle\sum_{i=1}^N C_{t,i}\,x_{h,t,i} \;+\; D_h\,u_{h,t}\Big).
\label{eq:output}
\end{equation}
The state matrix is diagonal and fixed; the token dependence enters only through
$B_t$, $C_t$, $\dt_{h,t}$, and the gate. Each channel is therefore a diagonal
realization of order $N$ with fixed continuous-time dynamics ($A_h$; the
discrete decay $\bar A = \exp(\dt A)$ still moves with the token) whose input
coupling, output coupling, and timestep are functions of the input: an
input-parameterized
(quasi-LPV) system, linear time-varying along each realized trajectory. The
channel is a bank of $N$ first-order modes, each a real pole $A_{h,i}<0$,
written and read on a per-token schedule.

\subsection{Exact per-mode decomposition}

Because $A_h$ is diagonal, the states $x_{h,t,i}$ do not interact across $i$, so
the output \eqref{eq:output} is an exact sum of per-mode contributions. Define
the post-gate output of mode $i$ in channel $h$,
\begin{equation}
m_{h,t,i} \;=\; g_{h,t}\,C_{t,i}\,x_{h,t,i},
\qquad
f_{h,t} \;=\; g_{h,t}\,D_h\,u_{h,t},
\label{eq:mode-output}
\end{equation}
so that $y_{h,t} = \sum_{i} m_{h,t,i} + f_{h,t}$ exactly. Pruning a mode means
omitting its contribution from the output. For a kept set
$S\subseteq\{1,\dots,N\}$, the pruned output is
$\hat y^{S}_{h,t} = \sum_{i\in S} m_{h,t,i} + f_{h,t}$, and the squared output
error over a window of length $L$ is a quadratic form in the dropped modes:
\begin{equation}
\sum_{t=1}^L \big(y_{h,t}-\hat y^{S}_{h,t}\big)^2
\;=\; \sum_{i\notin S}\ \sum_{j\notin S} G_{h,ij},
\qquad
G_{h,ij} \;=\; \sum_{t=1}^L m_{h,t,i}\,m_{h,t,j}.
\label{eq:gram}
\end{equation}
The Gram matrix $G_h\in\mathbb{R}^{N\times N}$ is symmetric positive
semidefinite; we store its upper triangle, $N(N{+}1)/2$ numbers per (layer,
channel, window). Equation~\eqref{eq:gram} gives the \emph{exact} output error
of pruning any subset $S$, in a single pass. For the nested family of
top-$r$ sets under a mode ordering $\pi$, the error at every budget $r$ is a
single reverse cumulative sum over $G_h$.

This is the object GHOST \citep{ghost2026} and activity criteria
\citep{perfmamba2025} approximate: they score modes once and keep a fixed set,
whereas \eqref{eq:gram} returns the exact cost of any set, resolved per input.

\subsection{Validation and layer-local cross-check}

The instrument recomputes each mixer's internal signals from its captured input
and reconstructs the output through the layer's own weights. We verify it in
two ways. First, reconstruction: on three windows by three layers per model, the
reconstructed mixer
output matches the reference implementation to a relative error of at most
$2.3\times10^{-7}$ on the Mamba-1 family, $2.6\times10^{-7}$ on Falcon-Mamba,
and $1.0\times10^{-6}$ on Mamba-2 (whose pre-norm decomposition, below, carries
one further gated normalization). Second, layer-local pruning: applying a mode
mask in the running layer and measuring the resulting internal SSM-output error
against the reference forward's own SSM output agrees with the value \eqref{eq:gram}
predicts offline to a median relative deviation of $1\times10^{-6}$ over
$4{,}464$ configurations (five models, every layer, six windows,
$r\in\{4,8,12\}$), maximum $3.6\times10^{-5}$; the same sweep bounds the
reconstruction error at the pre-projection output by $6.9\times10^{-6}$. The
pruned side necessarily runs through the per-mode decomposition (the reference
kernel has no per-channel mode mask), so the deviation floor is set by that
reconstruction error; the Gram identity itself is exact. This check is
layer-local and pre-projection. End-to-end changes in logits and perplexity are
measured separately in Section~\ref{sec:consequence}.

\subsection{Measures of state usage}

From the Gram tensor and the per-mode energies $e_{h,i}=G_{h,ii}$ we read off,
per layer:

\emph{Migration gap.} Fix a budget $r$ and a per-window keep-set rule that
selects a kept set $S_w$ for each evaluation window $w$. Write the layer's
energy-weighted relative output error over the evaluation set as
\begin{equation}
E(r) \;=\; \sqrt{\frac{\sum_w \sum_h \sum_{i,j\notin S_w} G_{h,ij}^{(w)}}
                     {\sum_w \sum_h \sum_t (y_{h,t}^{(w)})^2}},
\label{eq:Er}
\end{equation}
the total dropped energy \eqref{eq:gram} over the total output energy, both
summed over channels and windows. Let $E_{\mathrm{static}}(r)$ use one fixed set
$S$ ranked by mean per-mode energy on disjoint calibration windows, and
$E_{\mathrm{diag}}(r)$ use the per-window top-$r$ set $S_w$ by that window's own
diagonal per-mode energies $e_i=G_{ii}$. Their ratio
$\rho(r)=E_{\mathrm{diag}}(r)/E_{\mathrm{static}}(r)$ measures the excess error
of this calibration-ranked static set relative to the same-window diagonal
top-$r$ selector: $\rho\approx1$ means the diagonal-energy ranking is stable, and
small $\rho$ means that it changes with the input. The static set is ranked on
disjoint calibration windows while the per-window set reads the evaluation
window's own energies, so $\rho$ is an in-sample headroom ratio. It is not the
solution of the full-Gram subset-selection problem. The churn metric, the
frozen-$\Bt$ result, and the switched-system experiment below establish that the
gap reflects genuine cross-context migration.

The gap has a closed form that identifies the source of a static set's excess
error. Fix a
channel $h$ and write $e^w_i=G^{(w)}_{h,ii}$ for its per-mode diagonal energy in
window $w$; neglecting the off-diagonal Gram terms, dropping the set
$D\subseteq\{1,\dots,N\}$ costs $\widehat E^2(D)=\sum_{i\in D}e^w_i$, and
$\widehat E^2_{\mathrm{static}}(w)$, $\widehat E^2_{\mathrm{diag}}(w)$ denote
this cost for $D$ the complement of $S$ and of $S_w$ respectively.

\begin{proposition}\label{prop:churn}
For each channel and window, with $S$ the static top-$r$ set and $S_w$ the
per-window top-$r$ set by $e^w$,
\[
\widehat E^2_{\mathrm{static}}(w)-\widehat E^2_{\mathrm{diag}}(w)
= \sum_{i\in S_w\setminus S} e^w_i - \sum_{i\in S\setminus S_w} e^w_i \ \ge\, 0,
\]
the net energy of the modes that migrate into the top-$r$ set.
\end{proposition}

\begin{proof}
The two dropped sets share the complement of $S\cup S_w$, so their costs differ
only on the symmetric difference, giving the identity. For the inequality, both
sets keep $r$ modes, so $|S_w\setminus S|=|S\setminus S_w|$; and since $S_w$
holds the $r$ largest energies, $e^w_i\ge e^w_j$ for every $i\in S_w\setminus S$
and $j\in S\setminus S_w$. Pairing the two equal-size sets and summing makes the
difference non-negative.
\end{proof}

This establishes the diagonal $\rho\le1$: the top-$r$ churn $|S_w\setminus S|$
sets the support of the gap, and the energy of the churning modes sets its size.
The gap is governed by an energy-weighted churn. Two
measurements connect the theorem to the full-Gram quantity we report (the
inequality itself holds in every channel-window case, so it serves only as a
consistency check on the pipeline). Across layers, the diagonal gap matches the
full-Gram $\rho$ at correlation $0.96$, and the energy-weighted churn predicts
$1-\rho$ at Spearman $0.80$ against $0.74$ for raw set churn. The full-Gram gap
adds only the off-diagonal cross-term difference; it is small except in the
static tail, where it can push the measured $\rho$ slightly above one (layers
21--22 of 130M reach $1.003$); the theorem covers only the diagonal quantity.

\emph{Churn.} For two windows, we average the top-$r$ intersection size across
channels and convert that average intersection to a Jaccard distance. We report
this distance for within-domain and cross-domain pairs, with the same statistic
between the two halves of one window as a null (median $0.22$ at 130M);
migration appears as cross-domain churn above the null (median $0.29$).
The margin is modest: cross-domain churn exceeds within-domain in every layer,
but the pilot ratio was $1.28\times$, below the $1.5\times$ we pre-registered.
Churn is corroborating evidence; the migration case rests on the gap $\rho$,
the frozen-signal counterfactuals, and the switched-system experiment.

\emph{Effective order.} The participation ratio
$\mathrm{PR}_h=(\sum_i e_{h,i})^2/\sum_i e_{h,i}^2$, the number of modes that
carry the channel's energy: the effective order of the realization at the
operating point.

\subsection{Mamba-2}

In Mamba-2 \citep{dao2024mamba2} the state decay is scalar per head, so the
per-mode decomposition \eqref{eq:mode-output} holds at the pre-norm SSM output;
the gated RMSNorm that follows couples modes within a position, so the exact
offline-Gram identity \eqref{eq:gram} does not carry past the norm. We therefore
compute the same per-mode outputs and the migration gap online, per window, and
measure end-to-end pruning error empirically (Section~\ref{sec:consequence}).
The state size $N=128$ also makes storing $G_h$ impractical, which the online
computation avoids. The instrument's exactness is thus specific to the
per-channel diagonal case (Mamba-1, Falcon-Mamba); the state-usage measures
transfer.

\section{Migration, Mechanism, and Consequence}\label{sec:experiments}

\paragraph{Setup.}
We audit the released Mamba-1 checkpoints (130M, 370M, 790M, 1.4B, 2.8B), the
7B Falcon-Mamba \citep{falconmamba2024}, and two Mamba-2 checkpoints (130M,
780M). Windows are 1024 tokens, 32 per domain from three domains (English
prose, code, technical text); all seeds are fixed. Layerwise migration uses an
even/odd split of these 96 audit windows within each domain, with the static set
ranked on the even windows and evaluated on the odd windows. End-to-end
perplexity uses a later disjoint evaluation set of 24 windows per domain. Those
evaluation windows are disjoint at the window level, though a document may span
both sets; any residual leakage helps the static baseline, whose ranking comes
from the calibration windows, so the reported static-vs-same-window gaps are
conservative. Every reported number carries the reference-vs-code validation of
Section~\ref{sec:method}. Unless stated, the budget is $r=8$ of
$N=16$ for Mamba-1 and Falcon-Mamba, and $r=64$ of $N=128$ for Mamba-2
(one-half of the state).

\subsection{Input-driven mode migration}\label{sec:migration}

The set of modes carrying a layer's output changes with the input. At the most
affected layer of each model the migration gap reaches
$\rho\in[0.44,0.57]$: a same-window diagonal top-$r$ selector roughly halves the
output error incurred by the calibration-ranked static set at the same budget.
The effect is present at every scale of
the Mamba-1 family (minimum $\rho(8)$ between 0.51 and 0.54), in the deployed 7B
Falcon-Mamba (0.48), and in Mamba-2 (0.44--0.57). Cross-domain churn exceeds
within-domain churn in every layer of every model (368 of 368 layers across
the eight models), at a
modest median ratio of $1.15$--$1.32$ (Section~\ref{sec:method}).

The migration is localized in depth, and its geometry depends on the
architecture. In the small Mamba-1 models (130M--790M) the migrating layers form
a contiguous mid-network band at relative depth $0.6$--$0.75$, while the final
layers are static ($\rho\to1$); at 1.4B and above and in Falcon-Mamba the
strongly migrating layers are punctuated, located
early-to-mid network; Mamba-2 migrates throughout depth, with
architecture-dependent locations and no static tail (Figure~\ref{fig:band}). The effective order shows the same pattern:
the participation ratio rises from about two modes at the ends of the network to
five--eight in the migrating band. We report the contiguous-band geometry as a
Mamba-1 finding; the phenomenon itself is general.

\begin{figure}[t]
\centering
\includegraphics[width=\textwidth]{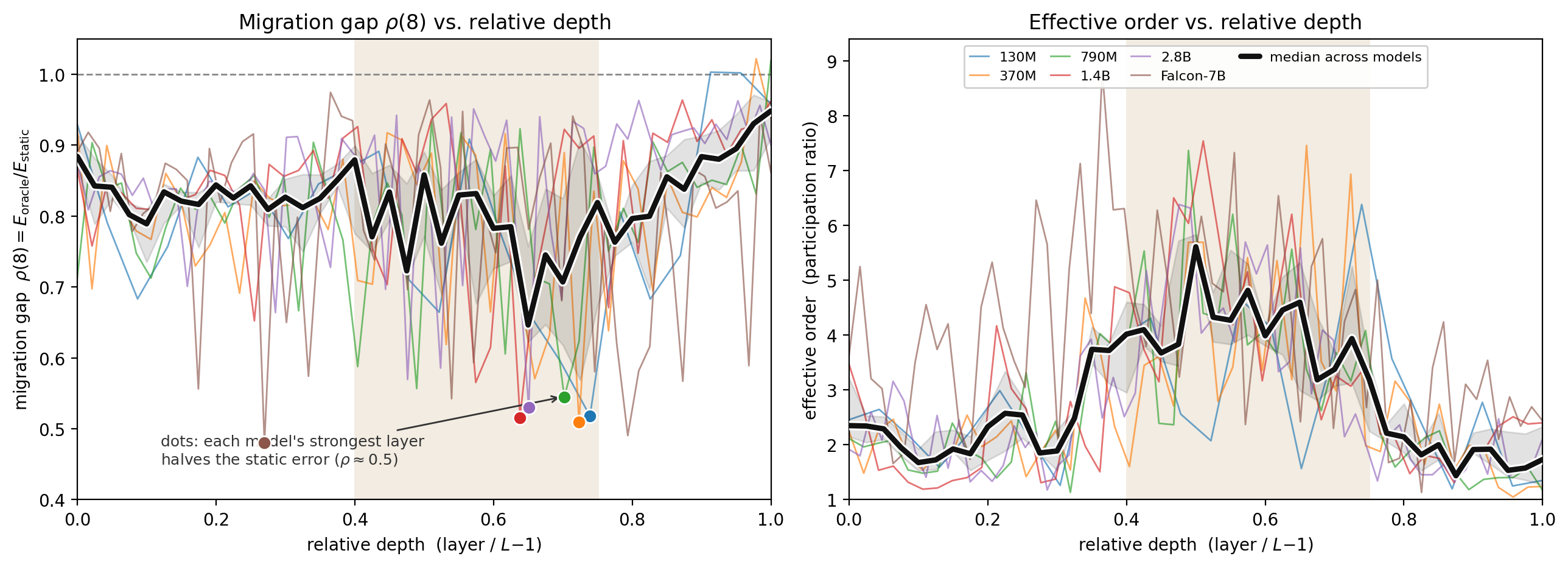}
\caption{Migration gap $\rho(8)=E_{\mathrm{diag}}/E_{\mathrm{static}}$
(left) and effective order (right) versus relative depth, over six models
(Mamba-1 130M--2.8B and Falcon-Mamba-7B). Coloured lines are individual
models (shared legend); the bold black line is the across-model median with an
interquartile ribbon; the shaded region is the migrating band (relative depth
$0.4$--$0.75$). Left: $\rho<1$ (dashed) means the same-window diagonal top-$r$
selector reconstructs the layer output with less error than the
calibration-ranked static selector; the median dips
through the band and each model's strongest layer, marked by a dot in the
model's colour, reaches $\rho\approx0.5$: the static set's error is
twice the same-window selector's at the same budget. Right: the participation
ratio rises from about two modes at the ends of the network to five--eight in
the same band, so migration concentrates in the layers that spread their
output over more modes.}
\label{fig:band}
\end{figure}

\subsection{Mechanism}\label{sec:mechanism}

Three token-dependent signals could account for the migration, each with a
distinct role in the realization. The write map $\Bt$ sets the instantaneous input
coupling, which modes the current token writes into; the readout $\Ct$ sets the
instantaneous output coupling, which modes the output reads; and the timestep
$\dt$ enters twice in \eqref{eq:recurrence}, as the effective poles
$\bar A=\exp(\dt A)$ and as a gain on the write term. Within a channel, however,
$\dt_{h,t}$ is common to all $N$ modes, so its write-gain role is mode-uniform:
of the three signals, only $\Bt$ (and $\Ct$ at readout) can differentially
select among modes at a token, while $\dt$ differentiates them only through the
exponent $\exp(\dt A_{h,i})$. We isolate each by rerunning a
layer's per-mode decomposition with one signal frozen at its corpus mean and
measuring the migration gap under the counterfactual (Table~\ref{tab:mechanism},
Figure~\ref{fig:mechanism}).
Freezing $\Bt$ removes the migration: the band-median gap rises from $0.70$ to
$0.96$ on Mamba-1 and from $0.77$ to $0.96$--$0.97$ on Mamba-2, and cross-domain
churn collapses. Freezing $\Ct$ removes part of it; freezing $\dt$ (both roles at
once) leaves the gap and the churn essentially unchanged. The migration is
therefore an input-coupling effect: the write map selects, token by token, which
modes the signal drives, and freezing it pins that selection to a single average
pattern on both architectures. The readout coupling $\Ct$ supplies the
remainder; the timestep $\dt$, the basis of activity pruning criteria, does
not carry the migration,
consistent with its mode-uniform write gain, which leaves the poles as its only
mode-differential action. Two controls confirm the attribution. Keeping only $\Bt$ live (freezing both $\Ct$
and $\dt$) reproduces much of the migration on its own ($\rho=0.87$ against the
real $0.70$ and the frozen $0.96$), so $\Bt$ is both necessary and largely
sufficient. And matching the frozen-$\Bt$
layer's total output energy to the real layer's leaves the gap unchanged
($0.964\to0.968$), which rules out an energy-renormalization artifact: the
$\rho\to1$ under freeze-$\Bt$ is a loss of context-dependence.

\begin{table}[t]
\centering
\caption{Migration gap under frozen-signal counterfactuals: Mamba-1 median
over the migrating band, Mamba-2 median over all layers (it migrates
throughout depth, with no static tail to exclude).
Freezing the write map $\Bt$ removes most of the migration ($\rho\to0.96$);
freezing the timestep $\dt$ does not. On Mamba-1, where we verified it, each
frozen variant's total output energy stays within an order of magnitude of the
real layer's, and the energy-matched control (Section~\ref{sec:mechanism})
rules out a renormalization artifact.}
\label{tab:mechanism}
\begin{tabular}{lccc}
\toprule
& Mamba-1 130M & Mamba-2 130M & Mamba-2 780M \\
\midrule
real            & 0.70 & 0.77 & 0.77 \\
freeze $\Bt$    & \textbf{0.96} & \textbf{0.96} & \textbf{0.97} \\
freeze $\Ct$    & 0.84 & 0.90 & 0.90 \\
freeze $\dt$    & 0.72 & 0.74 & 0.77 \\
\bottomrule
\end{tabular}
\end{table}

\begin{figure}[t]
\centering
\includegraphics[width=0.9\textwidth]{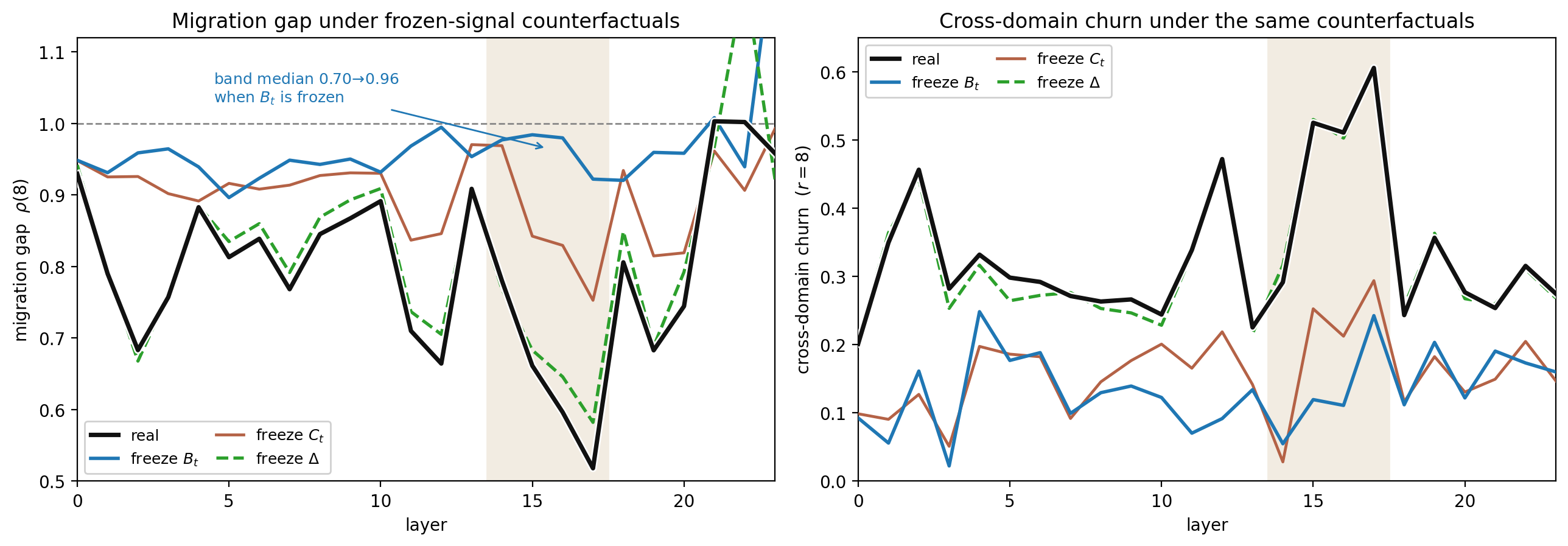}
\caption{Frozen-signal mechanism at Mamba-1 130M: migration gap $\rho(8)$
(left) and cross-domain churn (right) per layer, each selective signal frozen
at its corpus mean. Real (black); freeze $\Bt$ (blue); freeze $\Ct$ (rust);
freeze $\dt$ (green dashed). The shaded region is the migrating band (layers
14--17). Freezing the write map $\Bt$ lifts the band-median gap from $0.70$ to
$0.96$; freezing the readout $\Ct$ lifts it partway ($0.84$); freezing the
timestep $\dt$ overlays the real curve ($0.72$). Freezing $\Bt$ or $\Ct$
collapses the cross-domain churn, while freezing $\dt$ leaves it unchanged, so
the timestep carries neither the migration nor the churn.}
\label{fig:mechanism}
\end{figure}

Consistent with the mechanism, the migrating modes are the slow modes. Ordering
each mode by its trained time constant $1/(|A_{h,i}|\,\bar\dt_h)$, the
instability of top-$r$ membership across contexts rises with the
time constant from its minimum near one token toward the slow modes, reaching
$0.78$ in the slowest decile against a baseline near
$0.3$, at both 130M and 790M (Figure~\ref{fig:timescale}). Slow modes integrate
the input-gated writes over long horizons, so the direction a slow mode holds,
and hence which slow modes the input reaches, is whatever the context has written.

\begin{figure}[t]
\centering
\includegraphics[width=0.9\textwidth]{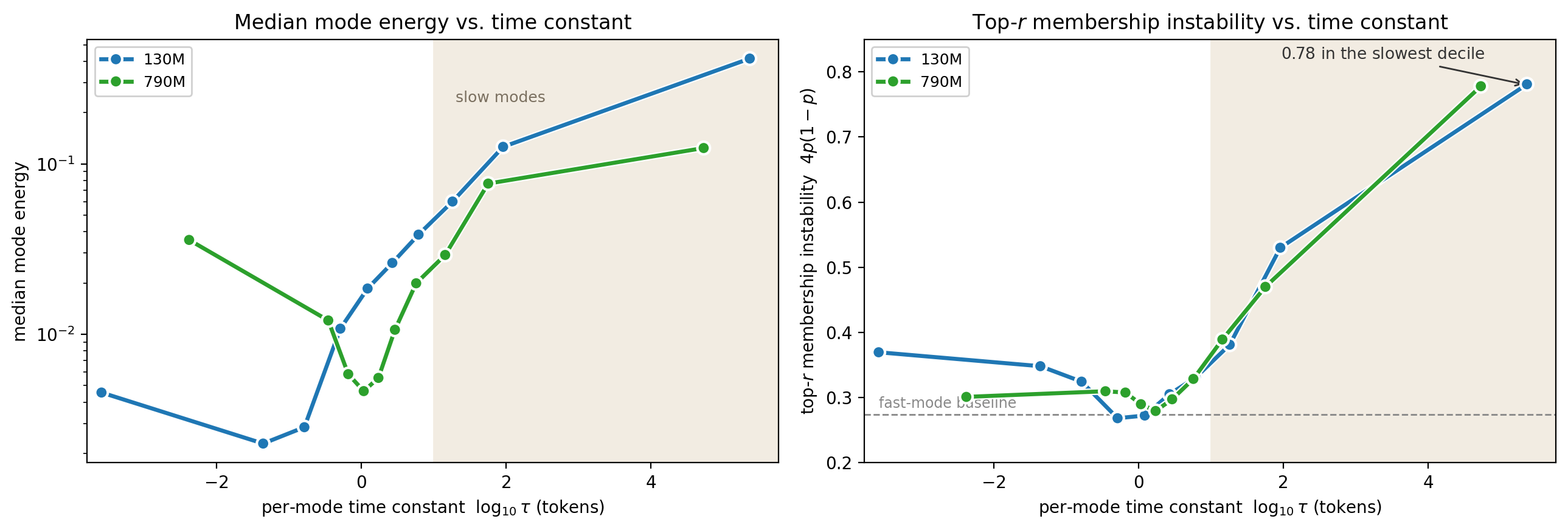}
\caption{Timescale structure at Mamba-1 130M (blue) and 790M (green): median
mode energy (left) and top-$r$ membership instability $4p(1{-}p)$ (right)
versus per-mode time constant $\tau=1/(|A_{h,i}|\,\bar\dt_h)$, binned into
deciles. Shaded: the slow modes ($\tau>10$ tokens). Both quantities rise
toward the slow modes: median energy grows by over an order of magnitude, and
the instability of top-$r$ membership rises from a fast-mode baseline near
$0.28$ to $0.78$ in the slowest decile, in both models.}
\label{fig:timescale}
\end{figure}

\subsection{Consequence}\label{sec:consequence}

If the important set migrates, then a selector with access to the input can have
lower error than a static choice at equal budget. The per-layer error
\eqref{eq:gram} is exact; the
end-to-end perplexity below is measured empirically, since pruning shifts the
inputs of downstream layers and the per-layer identity does not compose. We
compare, at each budget, three static rules and the scheduler: a calibration
energy ranking (\textbf{static}); a modal Hankel ranking that scores
each mode by the Hankel singular value built from its empirical per-mode
Gramians (realized write and readout second moments accumulated over the
calibration windows), instantiating the output-aware / Hankel-energy
scoring principle of GHOST and LAST on our per-mode outputs
(\textbf{modal-HSV}); a
layer-adaptive allocation of the same scores with the layer-wise cumulative
normalization of \citet{gwak2024last}, which keeps the
same total budget but distributes the kept modes non-uniformly across layers
(\textbf{LAST}; we normalize the realized-Gramian score rather than the
LTI $H_\infty$ norm of \citealp{gwak2024last}, which is not defined for a
token-varying realization; that method was designed for time-invariant
S4/S5 stacks); and a retrospective same-window input-scheduled set
(\textbf{scheduled oracle}: a first pass measures the scored window's mode
energies, and a second pass prunes that same window to its top-$r$ modes). The
scheduled-oracle column quantifies headroom. It is not a causal deployment
protocol, because its mask is allowed to depend on tokens in the scored window.
Table~\ref{tab:e2e} and Figure~\ref{fig:e2e} report this retrospective
perplexity at half the state budget across the Mamba-1 family (130M--2.8B), the
deployed 7B Falcon-Mamba, and Mamba-2. For the Mamba-1 and Falcon rows we report
$95\%$ paired CIs from an iid bootstrap over evaluation windows; the Mamba-2
rows were run without per-window CIs.

\begin{table}[t]
\centering
\caption{Retrospective same-window perplexity at half the state budget ($r=8$ of $16$; Mamba-2
at $r=64$ of $128$), with the paired same-window-oracle-vs-unpruned gap's
$95\%$ iid-window bootstrap CI ($10{,}000$ resamples) for Mamba-1 and Falcon.
The same-window oracle reads its mask from the window being scored, so this
column is an information-advantaged upper bound on input-dependent selection.
That upper bound lies below every static rule on every model where that rule is
run (modal-HSV and LAST are implemented for the per-channel diagonal case,
hence absent on Mamba-2), and its gap below the unpruned model has an iid-window
CI excluding zero from 130M to the deployed 7B; 370M was audited but not
included in this grid. Baselines instantiate the published scoring principles
on our per-mode outputs with no per-method tuning; scheduled uses none either.
A native comparison against the released GHOST code on Mamba-2 is reported
separately in Table~\ref{tab:ghost}.}
\label{tab:e2e}
\begin{tabular}{lcccccc}
\toprule
model & unpruned & static & modal-HSV & LAST & same-window oracle & oracle gap CI \\
\midrule
130M       & 12.38 & 14.16 & 14.52 & 25.78 & \textbf{11.84} & $[-0.66,-0.44]$ \\
790M       & 8.05  & 8.78  & 9.29  & 12.23 & \textbf{7.75}  & $[-0.36,-0.24]$ \\
1.4B       & 7.52  & 8.25  & 8.49  & 8.62  & \textbf{7.31}  & $[-0.27,-0.17]$ \\
2.8B       & 6.82  & 7.43  & 7.64  & 8.41  & \textbf{6.66}  & $[-0.21,-0.13]$ \\
Falcon-7B  & 4.48  & 4.82  & 4.96  & 5.27  & \textbf{4.44}  & $[-0.06,-0.03]$ \\
Mamba-2 130M & 12.30 & 13.58 & --  & --    & \textbf{11.00} & --             \\
Mamba-2 780M & 8.07  & 8.53  & --  & --    & \textbf{7.33}  & --             \\
\bottomrule
\end{tabular}
\end{table}

At the quarter budget $r=4$ the same ordering holds with larger gaps: at 130M,
static reaches perplexity $19.8$ and modal-HSV $20.8$, LAST collapses to $234$
(its layer-adaptive allocation strips some small-model layers below
viability), while scheduled
stays at $14.2$ (gap CI $[+1.54,+2.16]$ above the unpruned $12.4$). A
timestep-activity ranking and random mode sets, the two rules the mechanism
predicts should fail, collapse: at 130M, $r=8$ they reach perplexity $52.5$ and
$127$.

\begin{figure}[t]
\centering
\includegraphics[width=0.9\textwidth]{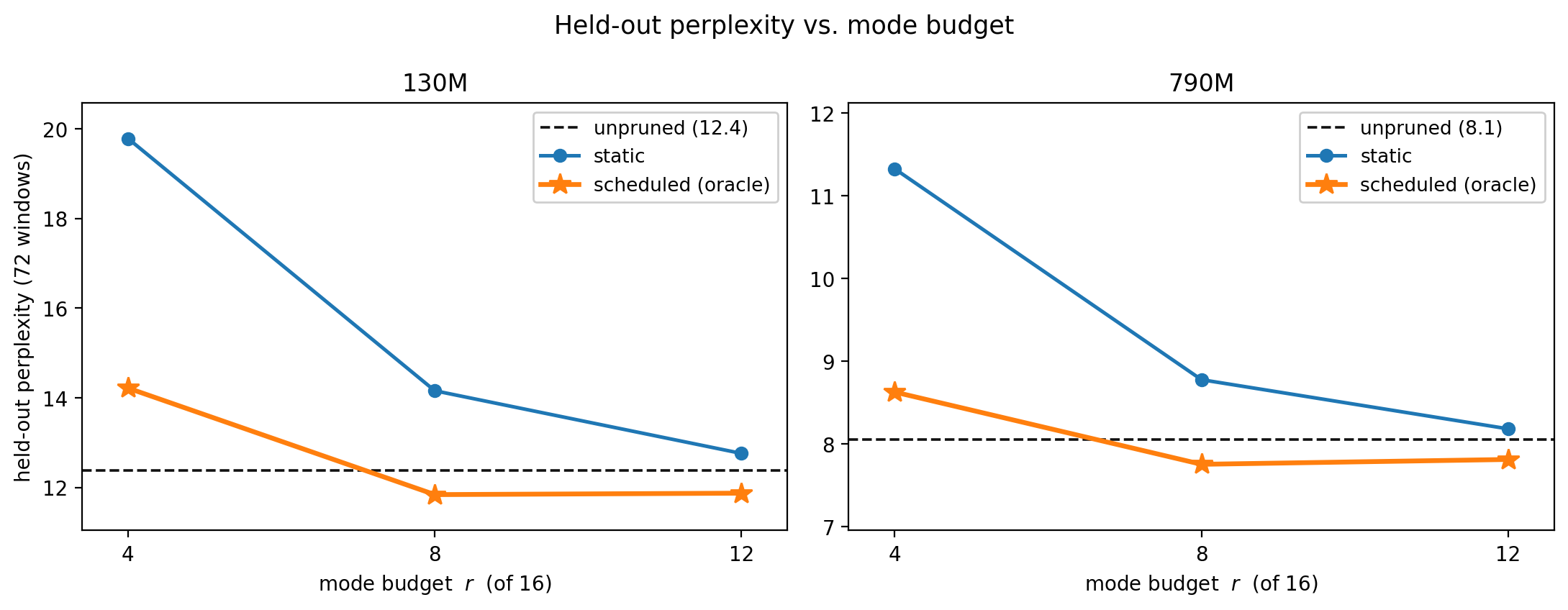}
\caption{Retrospective perplexity versus mode budget $r$ for Mamba-1 130M and 790M.
Input-scheduled selection (orange, a same-window oracle) is below static (blue)
at every budget, and at or below the unpruned baseline (dashed) for $r\ge8$;
the gap between static and scheduled widens as $r$ decreases (at $r=4$, static
$19.8$ versus scheduled $14.2$ on 130M). The mechanism-predicted rules are off
scale and omitted: timestep-activity ($\dt$-activity) and random reach
perplexity $52.5$ and $127$ at $r=8$ on 130M ($5.0\times10^3$ and
$9.1\times10^4$ at $r=4$).}
\label{fig:e2e}
\end{figure}

The same-window upper bound lies below every static rule it
is compared against, at every scale: at 7B it is $0.38$ perplexity below the
static energy ranking and further below modal-HSV and LAST, and this gap remains
large relative to the unpruned margin as scale grows (Table~\ref{tab:e2e}). The
comparison measures the value of information unavailable to a fixed static set.
Separately, the same-window oracle
also falls below the \emph{unpruned} model at half budget, with a paired
bootstrap CI excluding zero under iid window resampling at every Mamba-1 and
Falcon scale we test; the
margin shrinks with scale, from $0.54$ perplexity at 130M to $0.04$ at 7B. The
window bootstrap treats contiguous same-stream windows as exchangeable; a
circular block bootstrap (length-4 blocks within domain), which respects
within-stream autocorrelation, widens the 7B gap CI to $[-0.073,-0.020]$ and
still excludes zero. A
control identifies the source of this margin: selecting the mask on the first half of
each window and scoring the second half does \emph{not} outperform the unpruned model,
while the same-window mask still does under the same continuation scoring. The
sub-unpruned margin is therefore same-window adaptation: the mask must be
read from the tokens being scored. The control rules out a transferable
denoising effect. It also means the sub-unpruned result should be read as
retrospective headroom, not causal perplexity. The
same ordering holds on Mamba-2 (Table~\ref{tab:e2e}).

We also ran the released GHOST code natively on Mamba-2 at the same $50\%$
budget (Table~\ref{tab:ghost}). Under this table's protocol its perplexity
exceeds the realized-energy static ranking's at both scales and lies $2.98$
(130M) and $1.44$ (780M) above the same-window oracle. Because GHOST's importance score is an
output-aware realized-energy quantity, the comparison could turn on
calibration data; we therefore recomputed it with
GHOST calibrated on our own $96$ calibration windows, the same data the
static ranking uses. The result is nearly unchanged ($13.98\to13.86$, $8.77\to8.71$)
and remains above static, so the gap is a property of the
static-vs-input-scheduled distinction. On GHOST's own WikiText-2 protocol
the pruned models cost $+3.2$ and $+1.16$ word-perplexity; the 780M figure
matches its reported $\sim\!1$-point operating point, while 130M is more
sensitive to $50\%$ pruning at small scale. The released code's calibration
hooks attach to the reference forward, which \texttt{transformers} bypasses
whenever the fused Mamba kernels are installed, so on such machines the
released pipeline silently prunes nothing; we verified the no-op and forced
the reference path in a wrapper, leaving the method's code unchanged and
confirming that the resulting masks zero $50.0\%$ of the state at both scales
and both calibrations. The same-window headroom also appears on downstream
tasks where the state is used: on LAMBADA, static pruning to $r=8$ collapses
790M accuracy from
$0.63$ to $0.41$ and the same-passage mask (read from the scored passage)
recovers $86\%$ of the drop under the same information-advantaged protocol; on tasks that do
not stress long-range state (PIQA, ARC-easy) pruning to half budget has
almost no effect and the two rankings are within noise.

The scheduled result uses a full first pass to read the window's energies, and
the headroom it reaches is not cheaply captured. Under continuation scoring at
half budget, the fraction of the same-window oracle's gain captured grows monotonically with
how much of the scored window the mask has seen: a regime-conditioned mask (one
static set per domain, applied by the window's domain) captures $2$--$6\%$
(130M--7B), despite the domain being almost perfectly readable from the same
energies (the probe at the end of this section); a 256-token prefix estimator
captures $2$--$6\%$, and $5$--$9\%$ after a per-mode ramp-up correction learned
on calibration windows; selecting on the first half of the window captures
$18$--$26\%$ on the second. The capture fractions are protocol-sensitive:
scoring the continuation from position $128$ rather than from the prefix end
credits the same prefix estimator with $14$--$18\%$; we report the stricter
accounting. The important set moves within a regime and within a
window, decorrelating over a few hundred tokens, so the headroom is reached only
by measuring the tokens being scored. A scheduler that predicts the mask ahead
of the tokens it prunes remains future work; we do not claim one here.
Table~\ref{tab:oow} isolates the strongest out-of-window scheduler under a
protocol shared by all conditions: it recovers a quarter of the same-window
oracle's gain
over static without seeing the scored tokens, and it does not match the
unpruned model.

\begin{table}[t]
\centering
\caption{Out-of-window scheduling at half budget ($r=8$), continuation
perplexity on tokens $\geq512$ (all conditions share the scoring range). The
half-window mask is selected on each window's first 512 tokens and scored on
its second half: a deployable selection rule. It recovers
$25\%$ (130M) and $18\%$ (790M) of the same-window oracle's gain over static and
does not outperform the unpruned model.}
\label{tab:oow}
\begin{tabular}{lcccc}
\toprule
model & unpruned & static & half-window & same-window oracle \\
\midrule
130M & 10.37 & 12.19 & 11.69 & 10.14 \\
790M & 6.78  & 7.49  & 7.33  & 6.60  \\
\bottomrule
\end{tabular}
\end{table}

\begin{table}[t]
\centering
\caption{Native comparison against the released GHOST code on Mamba-2 at
$50\%$ state budget, evaluated under the Table~\ref{tab:e2e} protocol
(72 held-out windows). GHOST is run as released, with its own default
calibration (home) and re-run with calibration drawn from our $96$
calibration windows (ours), the same data the static ranking uses. Its
perplexity exceeds that of plain realized-energy static selection at both
scales under either calibration, and the same-window upper bound lies well
below it. Home-protocol
WikiText-2 costs: $+3.2$ (130M), $+1.16$ (780M) word-perplexity.}
\label{tab:ghost}
\begin{tabular}{lccccc}
\toprule
model & unpruned & static & GHOST (home) & GHOST (ours) & same-window oracle \\
\midrule
Mamba-2 130M & 12.30 & 13.58 & 13.98 & 13.86 & \textbf{11.00} \\
Mamba-2 780M & 8.07  & 8.53  & 8.77  & 8.71  & \textbf{7.33}  \\
\bottomrule
\end{tabular}
\end{table}

Finally, the phenomenon is not an artifact of text. A two-layer Mamba trained on
a synthetic switched linear system, with two known dynamical regimes, allocates
its modes by regime: a linear probe reads the active regime from the mode
energies at accuracy $1.00$, and cross-regime churn exceeds within-regime churn.
On the released language models the same probe reads the input domain from a
layer's mode energies at $99.7\%$ mean accuracy across all layers, against a
shuffled-label control near chance; the three domains are well separated, so this
shows the regime signature exists and is linearly readable; it does not show
that the discrimination is hard.

\paragraph{Reproducibility.}
All models are public released checkpoints, run in the reference fp32 path with
the fused kernels disabled so the instrument's reconstruction gate applies; the
only trained-by-us model is the disclosed two-layer switched-system Mamba, fully
specified with fixed seeds. Windows, calibration/evaluation splits, and all
random draws use seed $0$. All experiments ran on a single GB10 node; each phase
takes minutes to about an hour, and the full campaign is roughly one GPU-day.
Code (\codeurl) contains the scripts for each table, figure, and validation
gate; exact-number reproduction uses the released window manifests and result
summaries with the public checkpoints.

\section{Limitations}\label{sec:limitations}

The instrument is exact only for the per-channel diagonal case. Mamba-1 and
Falcon-Mamba have a distinct diagonal state matrix per channel and no mode
coupling after the SSM, so the Gram identity \eqref{eq:gram} gives the exact
error of any mode subset offline. Mamba-2 shares a scalar decay per head and
applies a gated normalization that couples modes within a position; there the
decomposition holds only at the pre-norm output and the end-to-end pruning error
is measured empirically.

The depth geometry is architecture-specific. The contiguous mid-network band
with a static tail is a Mamba-1 finding; the larger Mamba-1 models,
Falcon-Mamba, and Mamba-2 migrate mid-network but in punctuated layers. The
phenomenon, the mechanism, and the compression consequence hold across all of
them; only the band shape is specific.

The input-scheduled result is a measurement of headroom. Same-window scheduling
reads a window's mode energies from a full first pass over the scored tokens,
and therefore is not a causal deployment protocol and does not by itself
constitute a compute saving. Every
inexpensive scheduler we tried falls short: regime-conditioned masks recover
$2$--$6\%$ of the gain, a 256-token prefix estimator $2$--$9\%$ even with a
calibration-learned correction, and half-window selection $18$--$26\%$. The
important set decorrelates within a few hundred tokens, so the mask must be
measured on the scored tokens themselves. The sub-unpruned margin likewise does
not survive out-of-window selection. Predicting the mask before the tokens it
prunes remains an open problem. The
downstream evaluation covers standard tasks up to 1024 tokens; the migration is
strongest in the slow, long-horizon modes, so its effect on much longer contexts
is untested.

\section{Conclusion}\label{sec:conclusion}

A selective SSM layer is a bank of fixed diagonal modes, and the diagonal
structure makes state usage exactly measurable: one pass yields the exact output
error of pruning any mode subset, at any budget. The measurement shows that
trained selective SSMs re-allocate their state across inputs, that the
input-dependent write map $\Bt$ produces the re-allocation, and that a
same-window oracle gives an upper bound below every static mode ranking at
every scale and, at half the state budget, matches the unpruned model under a
retrospective protocol. The mask is read from the scored tokens, so this is
state-usage headroom; it is not a causal scheduler or a deployed saving. The phenomenon and
mechanism hold across the Mamba-1 family, the deployed 7B Falcon-Mamba, and
Mamba-2. The exact instrument replaces the heuristic choice of which
states to keep with a direct measurement whose answer varies with the
input. The realization-theoretic account of why the migration takes the form it
does, and a learned scheduler that exploits it cheaply, are the natural next
steps.

\bibliographystyle{plainnat}
\bibliography{references}

\end{document}